\documentclass{article}

\PassOptionsToPackage{numbers, compress}{natbib}
\usepackage[preprint]{neurips_2026}

\usepackage[utf8]{inputenc} 
\usepackage[T1]{fontenc}    
\usepackage{microtype}      
\usepackage[english]{babel}

\usepackage{xcolor}
\usepackage[colorlinks=true,
            linkcolor=blue,
            citecolor=blue,
            urlcolor=blue,
            allcolors=blue]{hyperref}

\usepackage{amsmath, amsthm, amssymb, mathtools, mathrsfs, bm, bbm, nicefrac}

\usepackage{graphicx, subfigure, wrapfig, float}
\usepackage{booktabs, threeparttable, multirow}
\usepackage{caption, enumitem, url}
\usepackage{algorithm}
\usepackage{algorithmic} 

\numberwithin{equation}{section}
\newtheorem{conj}{Conjecture}
\newtheorem{thm}[conj]{Theorem}
\newtheorem{cor}[conj]{Corollary}
\newtheorem{prop}[conj]{Proposition}
\newtheorem{lemma}[conj]{Lemma}
\newtheorem{ass}{Assumption}

\newtheorem{definition}{Definition}

\theoremstyle{remark}

\newtheorem{innercustomgeneric}{\customgenericname}
\providecommand{\customgenericname}{}
\newcommand{\newcustomtheorem}[2]{%
    \newenvironment{#1}[1] 
    {%
        \renewcommand\customgenericname{#2}%
        \renewcommand\theinnercustomgeneric{##1}%
        \innercustomgeneric
    }
    {\endinnercustomgeneric}
}
\newcustomtheorem{customAss}{Assumption}

\def\x{{\bm x}}
\def\X{{\bm X}}

\def\V{\bm V}

\def\0{\bm 0}
\def\1{\mathbbm 1}
\def\b1{\bm 1}

\title{Mixing Times of Glauber Dynamics on Masked Language Models}

\author{%
Suvadip Sana\thanks{Equal contribution.}\hspace{1.5mm}\thanks{Corresponding authors: \texttt{ss2776@cornell.edu}, \texttt{levine@math.cornell.edu}.}\\
Department of Statistics and Data Science\\
Cornell University\\
\texttt{ss2776@cornell.edu} \\
\And
Sami Wolf\footnotemark[1] \\
Cornell University\\
\texttt{smw362@cornell.edu}
\And
Neer Mehta\footnotemark[1] \\
Cornell University\\
\texttt{nmm229@cornell.edu}
\And
Alina Shah\footnotemark[1] \\
Cornell University\\
\texttt{ams877@cornell.edu}
\And
Aitzaz Shaikh\footnotemark[1] \\
Cornell University\\
\texttt{ams845@cornell.edu}
\And
Janna Goodman\footnotemark[1] \\
Cornell University\\
\texttt{jsg344@cornell.edu}
\And
Lionel Levine\footnotemark[2] \\
Department of Mathematics\\
Cornell University\\
\texttt{levine@math.cornell.edu}
}

\begin{document}

\maketitle

\begin{abstract}
  Masked language models (MLMs) define local conditional distributions over tokens but do not, in general, correspond to any consistent joint distribution over sequences. This raises a fundamental question: \textit{what global distributional behavior is induced when such conditionals are used iteratively for generation?} We address this question by modeling iterative masked-token resampling as a Glauber dynamics Markov chain on the discrete space of token sequences. We first show that MLM conditionals are intrinsically incompatible: we introduce a rectangle test that certifies this incompatibility and empirically verify its prevalence across modern MLMs. We then provide a theoretical analysis of the induced Markov chain. Under bounded cross-token influence, we establish a high-temperature contraction result implying $O(n\log n)$ mixing time where $n$ is the sequence length. In contrast, we prove that under a uniform local margin condition, the chain exhibits metastability, with exponentially slow escape from semantic basins at low temperatures.
   Empirically, we demonstrate a phase transition in mixing behavior as a function of temperature and sequence length, consistent with the theoretical predictions. We further characterize the induced stationary behavior through semantic trajectories, identifying persistent structures such as long-lived traps and recurrent semantic basins, with political content serving as a measurable case study.
\end{abstract}

\section{Introduction}
\label{sec:intro}

Masked language models (MLMs) such as BERT \cite{devlin2019bert} are commonly used for natural language understanding (NLU) tasks such as generating semantically meaningful embeddings \cite{devlin2019bert}, Machine translation \cite{conneau2019cross} sentiment analysis \cite{sun2019fine}, and named entity recognition \cite{devlin2019bert}. However, the global properties of MLMs, such as model biases and structural attractors, have not been completely characterized. Existing methods for interpreting MLMs fall short of capturing global structure. Single-token masked predictions ignore multi-token dependencies. Pseudo-Log-Likelihood scoring \cite{salazar2020masked} is widely used as if conditionals define a joint distribution; we show this is unjustified. Linear probes on static embeddings \cite{li2016visualizing, tenney2019you} miss temporal phenomena entirely, such as the metastable semantic traps we document in this work.

\begin{figure}[ht]
    \centering
    \includegraphics[width=0.8\textwidth]{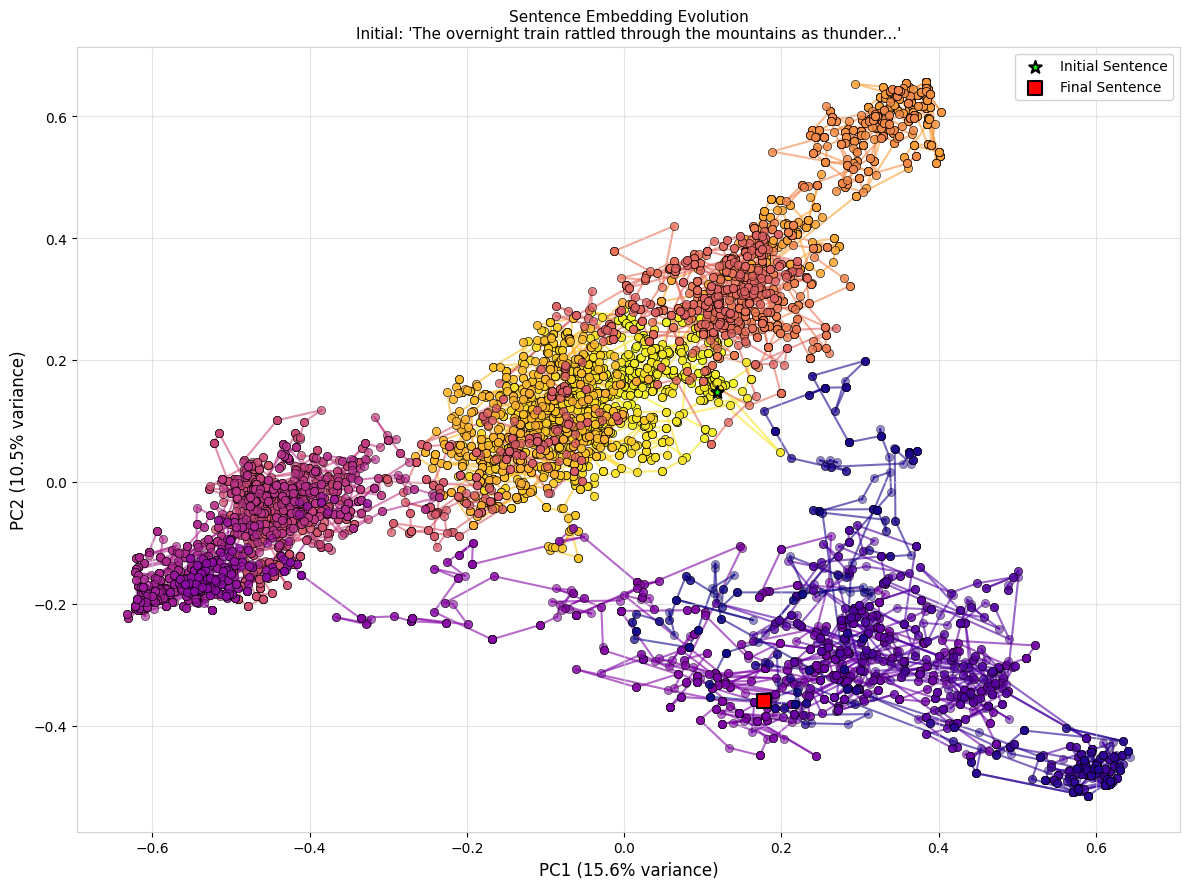}
    \caption{\textbf{Glauber dynamics on BERT exhibits metastable semantic basins.} PCA projection of sentence-embedding trajectories over 10,000 resampling steps, colored from warm (early) to cool (late). Tight clusters correspond to traps — configurations where the chain remains for hundreds to thousands of steps before escaping (\S \ref{sec:pca}). \textbf{Initial:} ``The overnight train rattled through the mountains as thunder echoed across the empty valley.'' \textbf{Final:} ``Greenville police say 2 sexual predators creeped on patients at Trinity Medical Center''}
    \label{fig:pca_10000}
\end{figure}

Mapping this landscape is necessary for reasons beyond interpretability. While MLMs are not historically used for standalone text generation, \cite{wang2019bert} showed that Gibbs sampling from these models can produce fluent text, and many modern discrete diffusion models involve MLMs during denoising \cite{li2022diffusion}. The energy landscape of the employed MLM is partially inherited by the generation process, raising the question of whether semantic biases in the MLM propagate to the outputs of these systems. As development of text-generating diffusion models continues to accelerate, understanding the hidden structure of these models will become increasingly important for fluent generation.

Glauber dynamics, a single-site Markov Chain Monte-Carlo (MCMC) method introduced to study the Ising model   \cite{glauber1963time}, provides the natural framework for this investigation: at each step, a random token position is masked and resampled from the MLM’s conditional distribution. Iterative resampling reveals two striking phenomena (Figs. \ref{fig:pca_10000},\ref{fig:phase_diagram}): chains become trapped in metastable semantic basins for thousands of steps, and the time to forget the initial state undergoes a temperature-dependent transition. We make four contributions:

\begin{enumerate}
    \item \textbf{Incompatibility.} We develop a “rectangle test” to measure the incompatibility of MLM conditionals by comparing the \textit{path-dependence} of multi-token swaps. We demonstrate that incompatibility is an intrinsic structural artifact of the pseudo-likelihood training objective: it persists across model sizes and is amplified by inter-token semantic influence. 
    \item \textbf{Mixing theory.} We establish provable, sufficient conditions for both fast and slow mixing parameterized by token influence functions. Specifically, when cross-site influence is bounded at high temperatures, we prove an $O(n \log n)$ upper bound on mixing time. Conversely, at low temperatures, we establish an exponentially slow lower bound in temperature, driven by uniform local score margins that trap the chain in metastable basins. We provide an empirically certified slow-mixing basin using a drift condition.
    \item \textbf{Phase transitions.} We empirically characterize how the boundary between fast and slow mixing depends on temperature and sequence length, validating the theoretical predictions on BERT.
    \item \textbf{Semantic landscape.} We probe the landscape induced by long-term dynamics to reveal semantic traps and recurrent basins. Using sentence embeddings, we track semantic trajectories to visualize metastable behavior. We analyze political content as a measurable case study of semantic recurrence across diverse chain initializations. We validate our approach by comparing the Glauber dynamics induced by several MLMs, including BERT \cite{devlin2019bert}, RoBERTa \cite{liu2019roberta}, and ModernBERT \cite{warner2025smarter}.

\end{enumerate}
\begin{figure}[ht]
    \centering
    \includegraphics[width=0.6\textwidth]{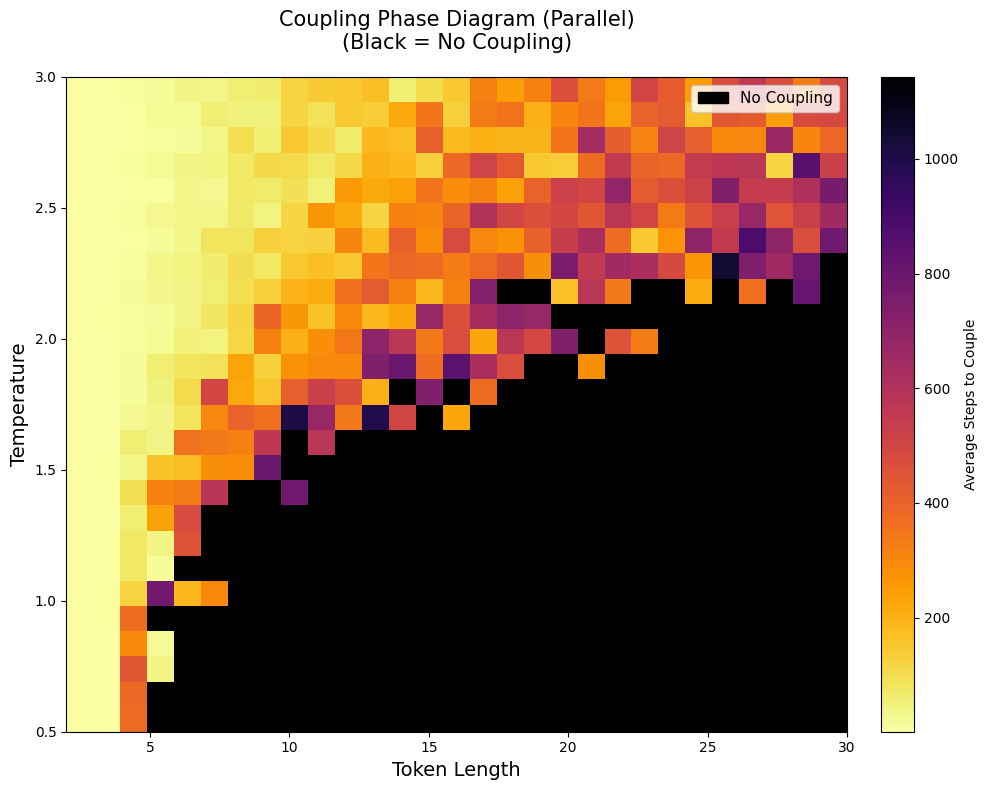}
    \caption{\textbf{A temperature-length phase transition in mixing time.} Two chains initialized from independent MS MARCO 
    passages on RoBERTa-base are evolved under maximal coupling. Color: median steps to coupling within a $10^4$-step budget. 
    Black: no coupling within budget. The slow-to-fast boundary near $\tau\approx$ 1.5-2 matches the regimes characterized in \S \ref{sec:theory}, \S \ref{sec:phase_transition}}.
    \label{fig:phase_diagram}
\end{figure}
\section{Related Work}
\label{sec:related}

\paragraph{Glauber dynamics and mixing in high-dimensional systems.}
Glauber dynamics was originally introduced in statistical physics as a stochastic single-site update process for sampling from Gibbs distributions, particularly in the Ising model \cite{glauber1963time}. At each step, one coordinate is selected and resampled conditioned on the remaining variables, producing a Markov chain whose stationary distribution corresponds to the target Gibbs measure. Theoretical analysis of Glauber dynamics and mixing times has played a central role in probability theory, statistical physics, and Markov Chain Monte Carlo methods \cite{levin2017markov}. 
\paragraph{Discrete diffusion and iterative denoising language models.}
Recent years have seen growing interest in diffusion and denoising approaches for discrete text generation. Structured discrete diffusion models \cite{austin2021structured} generalized denoising diffusion probabilistic models to discrete state spaces through categorical corruption processes. Diffusion-LM \cite{li2022diffusion} demonstrated controllable text generation using diffusion dynamics in continuous latent spaces, while DiffuSeq \cite{gong2022diffuseq} extended diffusion methods to sequence-to-sequence generation tasks. Several recent works established stronger connections between diffusion modeling and masked language modeling. DiffusionBERT \cite{he2023diffusionbert} integrated diffusion-style denoising with MLM objectives, while SEDD \cite{lou2023discrete} proposed ratio-based discrete diffusion modeling for text generation. Most closely related to our setting, masked diffusion language models \cite{sahoo2024simple} showed that iterative masked denoising objectives can serve as effective generative mechanisms for language modeling. Collectively, these works suggest that iterative local token refinement may serve as a viable alternative to autoregressive generation. Our work differs from prior diffusion-based language modeling research in that we focus not on generation quality or controllability, but on the global stochastic geometry induced by MLM conditionals. In particular, we study how repeated masked-token updates give rise to metastable semantic regions, topic persistence, and transition dynamics across sequence space.

\paragraph{Biases, probing, and representation analysis in language models.}
A large body of work investigates social biases and internal representations in language models. Early studies examined demographic and stereotypical biases in generated text by conditioning prompts on demographic identities and analyzing downstream sentiment or toxicity \cite{sheng2019woman}. Subsequent benchmarks such as CrowS-Pairs \cite{nangia2020crows}, StereoSet \cite{nadeem2021stereoset}, and RealToxicityPrompts \cite{gehman2020realtoxicityprompts} evaluated stereotypical associations and harmful generation behaviors in pretrained language models. Other lines of work focused on probing the internal structure of neural representations \cite{belinkov2022probing}. Our work differs from existing probing methodologies in two important ways. First, rather than studying static embeddings or one-step generations, we analyze the long-run stochastic dynamics induced by repeated masked-token updates. Second, instead of measuring isolated stereotypes or toxicity scores, we investigate global semantic behavior through topic-level metastability and transition structure in the induced Markov chain. This dynamical systems perspective connects language modeling with statistical physics and Markov chain theory, offering a new framework for understanding global structure and biases in masked language models.
\section{Method: Text Glauber Dynamics}
\label{sec:method}

\subsection{State space and temperature-scaled local conditionals}

Let $\V$ be a finite vocabulary and let
\(
\X_n := \V^n
\)
denote the set of token sequences of length $n$. Since $\V$ is finite, $\X_n$ is finite with $|\X_n| = |\V|^n$. When $n$ is fixed we write $\X$. For $x=(x_1,\dots,x_n)\in \X$ and $i\in[n]:=\{1,\dots,n\}$, write
\(
x_{-i} := (x_1,\dots,x_{i-1},x_{i+1},\dots,x_n).
\). We assume the masked language model induces, for each site $i$ and context $x_{-i}$, a strictly positive conditional distribution
\(
p_\theta(\cdot \mid x_{-i}) \in \Delta(\V).
\)
Given the temperature $\tau>0$, we assume there exist local scores
\(
s_i(a;x_{-i}) \in \mathbb{R}, a\in\V,
\)
that induce temperature-scaled conditionals
\[
p_{\theta,\tau}(a\mid x_{-i})
:=
\frac{\exp(s_i(a;x_{-i})/\tau)}
{\sum_{b\in\V}\exp(s_i(b;x_{-i})/\tau)}.
\]

\subsection{The Glauber Dynamics Chain}
Given $x\in\X$, define the single-site update chain with transition kernel
\[
P_\tau(x,x')
=
\frac1n \sum_{i=1}^n
\mathbf{1}_{x_{-i}=x'_{-i}} \,
p_{\theta,\tau}(x_i' \mid x_{-i}).
\]

Equivalently, from $x$ we choose a site $i$ uniformly from $[n]$ and resample $x_i$ from $p_{\theta,\tau}(\cdot\mid x_{-i})$.

\begin{prop}\label{prop:markovian}
For each $\tau>0$, $P_\tau$ is a Markov kernel on $\X$. If $p_{\theta,\tau}(a\mid x_{-i})>0$ for all $i,x_{-i},a$, then $P_\tau$ is irreducible and aperiodic.
\end{prop}

Detailed proofs for all Propositions, Lemmas, and Theorems are deferred to Appendix \ref{app:proofs}.

\begin{cor}
If $p_{\theta,\tau}(a\mid x_{-i})>0$ for all $i,x_{-i},a$, then there exists a unique stationary distribution $\mu_\tau$ on $\X$, and for every $x\in\X$,
\[
P_\tau^t(x,\cdot)\to \mu_\tau \quad \text{as } t\to\infty.
\]
\end{cor}

\section{Incompatibility of MLM Conditionals}
\label{sec:incompatibility}

We motivate Glauber dynamics as a probe for MLM structure by showing that
the relative probability of two sampling outcomes given the same context
need not match their relative prominence in the stationary distribution
$\mu_\tau$.

\begin{definition}\label{def:compatible}
A distribution $\pi$ on $\X$ is \emph{compatible} with the local conditionals
if, for every $i$, every context $x_{-i}$, and every $a \in \V$ with
$\pi(X_{-i} = x_{-i}) > 0$,
\[
  \pi(X_i = a \mid X_{-i} = x_{-i}) = p_{\theta,\tau}(a \mid x_{-i}).
\]
\end{definition}

\begin{thm}\label{thm:compatible}
If a compatible distribution $\pi$ exists, then $\pi$ is stationary for
$P_\tau$ and $P_\tau$ is reversible with respect to $\pi$.
\end{thm}
Compatibility implies single-mask probes read off stationary ratios directly:
for $x_a = (x_{-i}, a)$ and $x_b = (x_{-i}, b)$,
\begin{align}
  \frac{p_{\theta,\tau}(a \mid x_{-i})}{p_{\theta,\tau}(b \mid x_{-i})}
  = \frac{\pi(x_a)}{\pi(x_b)}.
  \label{eq:ratio}
\end{align}
Without compatibility, the biases revealed by single-mask probing may differ
from those exhibited under iterative generation. We now show no compatible
distribution exists for MLMs trained with a pseudo-log-likelihood objective.

\paragraph{The rectangle test.}
Fix a sequence $\x$ and two positions $i \neq j$, with current tokens
$A = x_i$, $B = x_j$ and replacements $A', B'$.  Define four states
differing only at positions $i$ and/or $j$:
\[
  \x = (\ldots,A,\ldots,B,\ldots), \;\;
  \mathbf{y} = (\ldots,A',\ldots,B,\ldots), \;\;
  \mathbf{w} = (\ldots,A,\ldots,B',\ldots), \;\;
  \mathbf{z} = (\ldots,A',\ldots,B',\ldots).
\]
Each edge has a log-probability ratio from the MLM, e.g.\
$s_{\x \to \mathbf{y}} = \log p_{\theta,\tau}(A' \mid \x_{-i})
                        - \log p_{\theta,\tau}(A  \mid \x_{-i})$.
We define the \textbf{rectangle incompatibility}
$\delta = (s_{\x\to\mathbf{y}} + s_{\mathbf{y}\to\mathbf{z}})
        - (s_{\x\to\mathbf{w}} + s_{\mathbf{w}\to\mathbf{z}})$.

\begin{prop}\label{prop:rectangle}
If the conditionals admit a compatible distribution, then $\delta = 0$ for
every rectangle.
\end{prop}
This must hold because each edge log-ratio equals a log-ratio of $\pi$.
Both paths from $\x$ to $\mathbf{z}$ telescope to
$\log \pi(\mathbf{z})/\pi(\x)$, so $\delta = 0$. Any $\delta \neq 0$ is therefore a certificate of incompatibility.

\paragraph{Results.}
We sample 300 random rectangles per model from MS MARCO text, with replacement tokens drawn from
each model's own top-50 predictions. As a control, we run the same test on
GPT-2: since it defines a joint distribution, each edge log-ratio is a
difference of joint log-probabilities, and $\delta = 0$ by telescoping.

\begin{table}[h]
\centering
\small
\caption{Rectangle incompatibility across models. All MLMs show significant
incompatibility ($p < 10^{-14}$); GPT-2 gives $\delta = 0$ exactly.}
\begin{tabular}{lcrr}
\toprule
Model & Type & Params & Mean $|\delta|$ \\
\midrule
BERT-base        & MLM & 110M & 0.644 \\
BERT-large       & MLM & 335M & 0.798 \\
RoBERTa-base     & MLM & 125M & 0.705 \\
RoBERTa-large    & MLM & 355M & 0.735 \\
ModernBERT-base  & MLM & 150M & 0.753 \\
\midrule
GPT-2            & AR  & 124M & 0.000 \\
\bottomrule
\end{tabular}
\label{tab:rectangle}
\end{table}

All five MLMs show incompatibility in the range $0.64$--$0.80$, persisting
across architectures and model sizes. We show in Appendix \ref{app:supp} that incompatibility $|\delta|$ has a broad distribution, proving it is the rule, not the exception (Figure \ref{fig:incompatibility_distribution}) and that it is amplified by inter-token semantic influence (Figure \ref{fig:incompatibility_influence}.) Since no compatible distribution
exists, the stationary distribution $\mu_\tau$ is an emergent object whose
properties cannot be directly inferred from the local conditionals alone.  We
characterize $\mu_\tau$ through its mixing behavior in the following sections.

\section{Mixing Bounds for Glauber Dynamics}
\label{sec:theory}
To characterize the long-term behavior of text-space Glauber dynamics, we prove that there exists an exponential-time slow-mixing regime at low temperatures and an $O(n \log n)$ fast-mixing regime at high temperatures. In this section, we parametrize each regime with token influence functions and MLM-assigned score differences, then show an empirically certified slow-mixing basin using a drift condition.

\subsection{High-temperature rapid mixing}

We now prove a sufficient condition for rapid mixing.
For probability measures $\nu,\nu'$ on $\V$, write
\[
\|\nu-\nu'\|_{\mathrm{TV}}
:=
\frac12 \sum_{a\in\V} |\nu(a)-\nu'(a)|.
\]

\begin{definition}[Influence coefficients]
For $i\neq j$, define
\[
c_{ij}(\tau)
:=
\sup
\left\{
\|p_{\theta,\tau}(\cdot\mid x_{-i})-p_{\theta,\tau}(\cdot\mid y_{-i})\|_{\mathrm{TV}}
:
x,y\in\X,\; x_k=y_k \ \forall k\neq j
\right\}.
\]
Let
\[
\alpha(\tau):=\max_{i\in[n]} \sum_{j\neq i} c_{ij}(\tau).
\]
\end{definition}

\begin{thm}[High-temperature contraction]\label{thm:contraction}
Assume $\alpha(\tau)<1$. Then for the Hamming metric
\[
d_H(x,y):=\sum_{i=1}^n \mathbf{1}_{x_i\neq y_i},
\]
there exists a coupling of the chain such that for all $x,y\in\X$,
\[
\mathbb{E}[d_H(X_1,Y_1)\mid X_0=x,Y_0=y]
\le
\left(1-\frac{1-\alpha(\tau)}{n}\right)d_H(x,y).
\]
Consequently,
\[
t_{\mathrm{mix}}(\varepsilon)
\le
\frac{n}{1-\alpha(\tau)}
\left(\log n + \log \frac1\varepsilon\right),
\]
where $t_{\mathrm{mix}}(\varepsilon)$ denotes total-variation mixing time.
\end{thm}

\subsection{A sufficient high-temperature criterion in terms of score oscillation}\label{sec:oscillation}

The previous theorem is stated in terms of influence coefficients. We now relate those coefficients to temperature.

For $i\neq j$, define the cross-site score oscillation
\[
\Delta_{ij}
:=
\sup_{\substack{x,y\in\X\\x_k=y_k,\ \forall k\neq j}}
\left[
\max_{a\in\V}\bigl(s_i(a;x_{-i})-s_i(a;y_{-i})\bigr)
-
\min_{a\in\V}\bigl(s_i(a;x_{-i})-s_i(a;y_{-i})\bigr)
\right].
\]

\begin{lemma}\label{thm:oscillation}
For all $i\neq j$ and $\tau>0$,
\[
c_{ij}(\tau) \le \frac{\Delta_{ij}}{4\tau}.
\]
\end{lemma}

\begin{cor}
If
\[
\max_i \sum_{j\neq i} \Delta_{ij} < 4\tau,
\]
then $\alpha(\tau)<1$ and hence
\[
t_{\mathrm{mix}}(\varepsilon)
\le
\frac{n}{1-\alpha(\tau)}
\left(\log n+\log\frac1\varepsilon\right).
\]
\end{cor}

\paragraph{On tightness.} The contraction theorem requires $\alpha(\tau) < 1$, a worst-case condition over both contexts and update sites, dominated by adversarial coherent contexts where BERT's conditionals are sharply peaked. Direct estimation of $\alpha(\tau)$ on BERT (Appendix A.1) gives $\alpha(10) \approx 1.04$ and $\alpha(11) \approx 0.94$, so the sufficient condition requires at least $\tau>10$. This is an order of magnitude above where BERT produces coherent text and where empirical mixing is fast (§6.1). The contraction argument is therefore not tight at natural temperatures.

\subsection{Low-temperature metastability and slow mixing}

We now give a rigorous slow-mixing theorem under a local stability assumption.

For each site $i$ and context $x_{-i}$, let
\[
m_i(x_{-i}) \in \arg\max_{a\in\V} s_i(a;x_{-i})
\]
be a choice of score maximizer.

Let $B\subseteq \X$ be a set of states, to be interpreted as a metastable basin.

\begin{ass}[Uniform local margin on a basin]\label{ass:margin}
There exists $\Delta_\star>0$ such that for every $x\in B$ and every site $i$ for which changing $x_i$ would move the chain outside $B$, one has:
\[
x_i = m_i(x_{-i}),
\]
and
\[
s_i(x_i;x_{-i}) - s_i(a;x_{-i}) \ge \Delta_\star
\qquad
\text{for all } a\neq x_i \text{ with } x^{(i\to a)}\notin B.
\]
Here $x^{(i\to a)}$ denotes the sequence obtained from $x$ by replacing coordinate $i$ by $a$.
\end{ass}

\begin{thm}[Low-temperature exponentially small escape]\label{thm:exponential}
Assume the uniform local margin condition on a basin $B$. Then for every $x\in B$,
\[
P_\tau(x,B^c)
\le
|\V| e^{-\Delta_\star/\tau}.
\]
Consequently the conductance of $B$ satisfies
\[
\Phi_\tau(B)
:=
\frac{\sum_{x\in B}\mu_\tau(x)P_\tau(x,B^c)}{\mu_\tau(B)}
\le
|\V| e^{-\Delta_\star/\tau}.
\]
If additionally $\mu_\tau(B)\le \frac12$, then
\[
t_{\mathrm{mix}}(1/4)
\ge
\frac{1}{4|\V|} e^{\Delta_\star/\tau}.
\]
\end{thm}

Note that if $\Delta_\star = O(n)$, then we will have exponential mixing in $n$ where $n$ is the length of the sequence. These are structural theorems about the induced text-space dynamics. Whether a specific masked language model such as BERT satisfies the hypotheses at a given temperature is an empirical question that we defer to Appendix \ref{app:measurements}.

\subsection{An empirically certified slow-mixing basin}

Theorem~\ref{thm:exponential} certifies slow mixing under a uniform local score-margin hypothesis. We now exhibit a concrete basin for which a related but distinct condition, positive drift on the boundary, can be verified empirically, providing an alternative slow-mixing certificate.

Define $B = \{x : \#\{i : x_i = \text{`!'}\} \geq 0.9N\}$ for $N = 100$. We sample 200 states with exactly $0.9N$ `!' tokens in random positions and compute the expected one-step change in `!'-count under the Glauber kernel,
\begin{equation*}
    \delta_\tau(x) \;:=\; \frac{1}{N}\!\left[\sum_{i:\, x_i \neq \text{`!'}} p_{\theta,\tau}(\text{`!'} \mid x_{-i}) \;-\; \sum_{i:\, x_i = \text{`!'}} \bigl(1 - p_{\theta,\tau}(\text{`!'} \mid x_{-i})\bigr)\right].
\end{equation*}
A positive $\delta_\tau(x)$ indicates the `!'-count process is, in expectation, pulled deeper into $B$ at $x$.

\begin{table}[h]
\centering
\small
\caption{Minimum and median one-step `!'-count drift $\delta_\tau$ across 200 boundary samples of $B$.}
\label{tab:bert_drift}
\begin{tabular}{lrr}
\toprule
\textbf{Temperature} & \textbf{Min.\ $\delta_\tau$} & \textbf{Median $\delta_\tau$} \\
\midrule
$\tau = 0.5$ & 0.042 & 0.080 \\
$\tau = 1.0$ & 0.024 & 0.066 \\
\bottomrule
\end{tabular}
\end{table}

All 200 samples exhibit strictly positive drift at both temperatures. Since the count $X_t = \#\{i : x_{t,i} = \text{`!'}\}$ changes by at most one per Glauber step, $X_t$ is a strict submartingale at the boundary of $B$, and a standard drift argument yields escape times exponential in $N$, certifying slow mixing for this basin. This is a drift-based certificate, separate from the conductance-based bound of Theorem~\ref{thm:exponential} which requires a uniform local score margin.
\section{Results}
\label{sec:results}
\subsection{Mixing-Time Crossover}
\label{sec:phase_transition}

We assess how mixing behavior depends on temperature and sequence length 
using two complementary diagnostics. At high temperatures, we run Glauber 
dynamics on BERT-base-uncased starting from MS MARCO passages and measure 
semantic drift via cosine distance $d(x_t, x_0) = 1 - \cos(\phi(x_t), \phi(x_0))$ 
in Jina Embeddings v4 space (Appendix~\ref{app:details}); we use the hitting 
time to $d = 1.0$ as our mixing diagnostic. The threshold $1.0$ sits below 
saturation in the fast-mixing regime (Figure~\ref{fig:mixing_evidence}a) and 
is rarely reached at low temperatures within our step budget; the qualitative 
shape of the results is robust to threshold choice in $[0.4, 1.0]$.

At sufficiently high temperatures ($\tau \gtrsim 1$), embedding distance 
grows logarithmically with sequence length at fixed time, and hitting times 
scale as $C(\tau)\, n \log n$ — consistent with the contraction theorem 
(Theorem~\ref{thm:contraction}). Figure~\ref{fig:mixing_evidence} documents 
this scaling at $\tau = 1.0$ (left, embedding saturation) and $\tau = 1.4$ 
(right, $n \log n$ fit to hitting times). The empirical fast-mixing boundary sits below the contraction threshold; see \ref{sec:oscillation}.

\begin{figure}[ht]
    \centering
    \subfigure[Embedding distance $d(x_T, x_0)$ vs.\ sequence length at $\tau = 1.0$, $T = 10^4$ steps (100 trials per length). Distance saturates logarithmically, consistent with fast mixing.]{
        \includegraphics[width=0.48\textwidth]{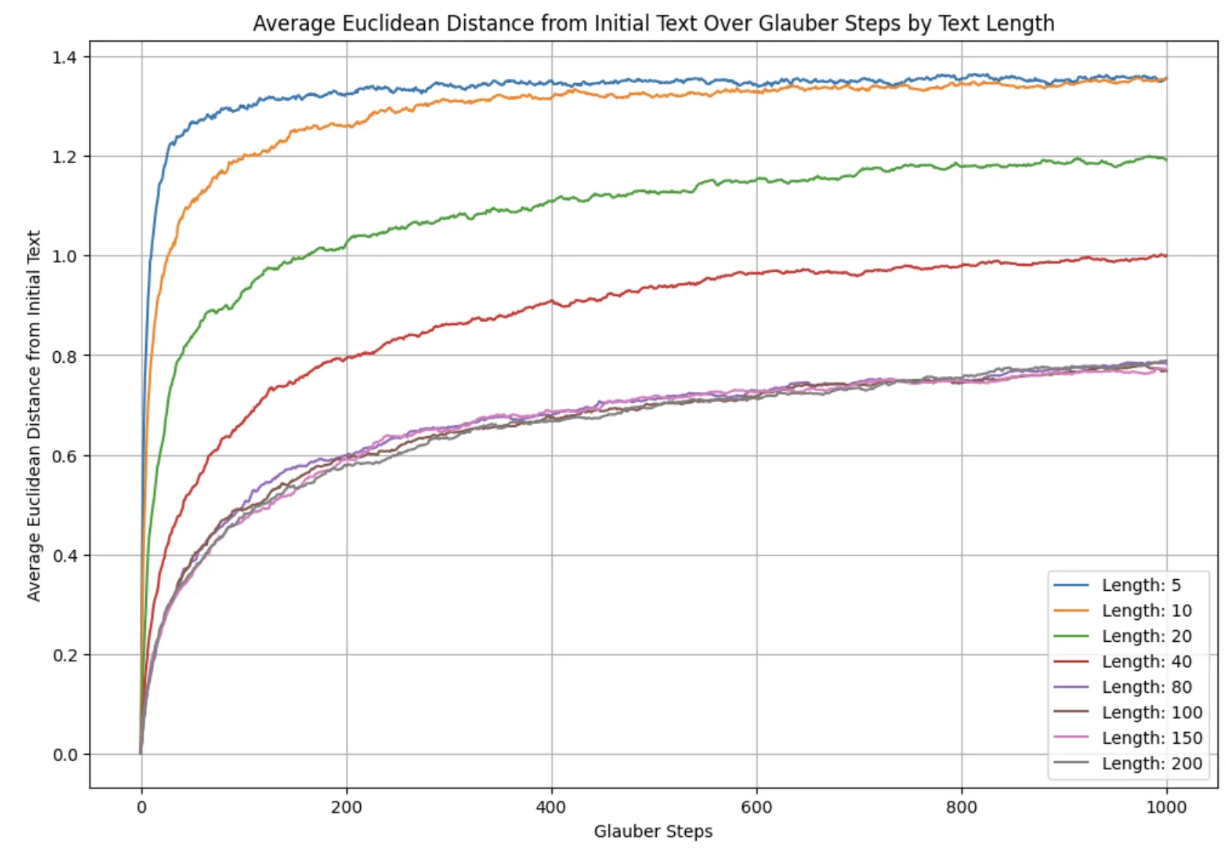
}
    }
    \hfill
    \subfigure[$C(\tau)\, n \log n$ fit to median hitting times at $\tau = 1.4$ (50 trials per length, cutoff 1000 steps). The fit is tight across all tested lengths.]{
        \includegraphics[width=0.48\textwidth]{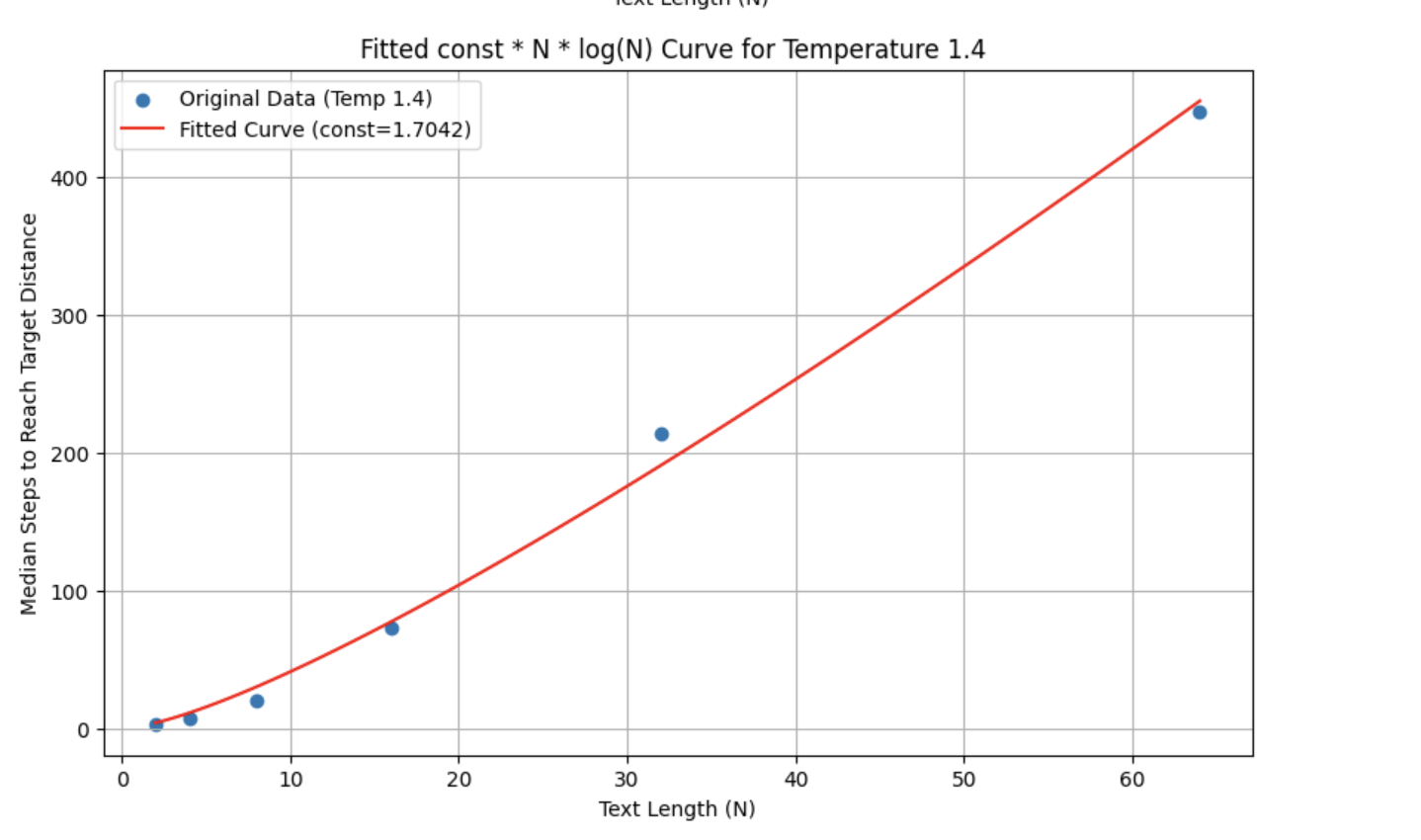}
    }
    \caption{Evidence for $C(\tau)\, n \log n$ mixing at high temperature on 
    BERT-base-uncased.}
    \label{fig:mixing_evidence}
\end{figure}

We further probe the mixing-time dependence on temperature and sequence length with a coupling mechanism. For each pair $(\tau,n)$ two chains initialized from independent MS MARCO passages are evolved under a maximal coupling at the same site, and we record the first step at which they agree. The transition from no-coupling-within-budget to 
fast coupling is centered around $\tau \approx 1.5$--$2$, with the boundary 
shifting to higher $\tau$ as $n$ grows. The mixing-time crossover qualitatively matches the regimes characterized in Section \ref{sec:theory}. However, the complex interactions of natural language induced by the MLM create a smoother transition boundary than exists in the physical Ising model for high $n$; see Section \ref{sec:limitations}. 

\subsection{Recurrent Semantic Basins in Long-Run Dynamics}
\label{sec:drift}

Long Glauber trajectories reveal a recurring phenomenon: chains often lose their initial
semantic identity and move through multiple semantically coherent
regions of text space. We refer to these regions as semantic basins.
Some basins are short-lived, while others persist for many updates or
are revisited repeatedly.

In this section, we use political content as a case study of such a
basin. This lets us ask
whether long-run masked-token dynamics merely diffuse away from the
initial prompt, or whether they repeatedly enter recognizable semantic
regions.

\subsubsection{Politics as a Recurrent Basin in BERT}

We ran 50 Glauber chains on \texttt{bert-base-uncased}
(5 topics $\times$ 10 seeds) for 10{,}000 steps at
$\tau=1.0$, recording topic scores every 200 steps with the
\texttt{facebook/bart-large-mnli} zero-shot classifier over nine
candidate labels. We report mean politics scores and original seed-topic
scores, which measure recurrence into a political semantic region and
loss of initialization-topic identity, respectively. The uniform-label
baseline is $1/9 \approx 0.111$.

\vspace{-0.5em}
\begin{table}[H]
\centering
\scriptsize
\caption{BERT political-topic recurrence and loss of initialization-topic
identity. Left: mean politics score at step 0 vs.\ step 10{,}000. Right:
politics score compared with the original seed-topic score. Baseline
$=0.111$.}
\label{tab:bert_drift_and_seed_loss}

\begin{minipage}{0.46\linewidth}
\centering
\textbf{(a) Politics score}\\[1mm]
\begin{tabular}{lrrr}
\toprule
\textbf{Topic} & \textbf{@0} & \textbf{@10k} & \textbf{$\Delta$} \\
\midrule
Health     & 0.008 & 0.044 & $+$0.036 \\
Sport      & 0.011 & 0.036 & $+$0.025 \\
Science    & 0.003 & 0.056 & $+$0.053 \\
Food       & 0.003 & \textbf{0.174}$\checkmark$ & $+$0.171 \\
Technology & 0.004 & 0.038 & $+$0.034 \\
\bottomrule
\end{tabular}
\end{minipage}
\hfill
\begin{minipage}{0.52\linewidth}
\centering
\textbf{(b) Politics vs.\ seed topic}\\[1mm]
\begin{tabular}{lrrrr}
\toprule
\textbf{Topic} & \textbf{Pol@0} & \textbf{Pol@10k} &
\textbf{Seed@0} & \textbf{Seed@10k} \\
\midrule
Health     & 0.008 & 0.044 & 0.631 & 0.116 \\
Sport      & 0.011 & 0.036 & 0.763 & 0.204 \\
Science    & 0.003 & 0.056 & 0.682 & 0.068 \\
Food       & 0.003 & 0.174 & 0.834 & 0.037 \\
Technology & 0.004 & 0.038 & 0.846 & 0.152 \\
\bottomrule
\end{tabular}
\end{minipage}
\end{table}

Every topic shows an increase in mean politics score. Food is the only
topic whose endpoint mean crosses the baseline, driven by two chains
ending in strongly political regions with scores $0.876$ and $0.497$.
However, endpoint scores understate the recurrence effect: 46 of 50
chains cross the politics baseline at least once, with transient peaks
as high as $0.979$.

As a robustness check, DeBERTa-v3 scoring gives the same qualitative
pattern: politics scores increase for all five topics, and 48 of 50
chains visit above-baseline political regions. Meanwhile, seed-topic
scores fall sharply in every case, showing that chains lose their
initial topical identity rather than remaining near the starting
semantic class.

We also compare endpoint political-content recurrence across
\texttt{bert-base-uncased}, \texttt{roberta-base}, and
\texttt{ModernBERT-base}; full endpoint values are reported in
Appendix~\ref{app:cross_model_drift}. RoBERTa shows the strongest
endpoint recurrence in this preliminary comparison, with two topics
crossing the baseline, while ModernBERT shows the weakest endpoint
recurrence, with no topics crossing the baseline. The observed
long-lived content also differs by model: RoBERTa chains tend to settle
into news-like or multilingual political fragments, while ModernBERT
chains more often enter instruction-style prompts or code-like regions.

\subsubsection{Interpretation}

These results are best understood as evidence for metastable semantic
basins rather than directed movement toward politics. Once chains lose
their initial topical identity, they explore multiple basins in the
model-induced text space. Political content is one interpretable basin
that is often reached, but it is not an absorbing endpoint of the
dynamics.

Persistence depends on a basin's effective depth: regions with larger
local score margins or fewer likely escape moves are harder to leave.
Thus, high transient politics scores and occasional political endpoints
suggest recurrent entry into a political semantic region. More broadly, Glauber dynamics
serves as a diagnostic for the semantic basin structure induced by
masked language models.

\section{Discussion, Conclusion and Limitations}
\label{sec:limitations}
\paragraph{Conclusion.} We have studied the Markov chain induced by iterative masked-token resampling
in masked language models through the lens of Glauber dynamics. The rectangle
test certifies that MLM conditionals are fundamentally incompatible with any
joint distribution, a structural consequence of pseudo-log-likelihood training
that persists across architectures and scales. Under bounded cross-token
influence we establish an $O(n \log n)$ mixing time upper bound; under a uniform
local margin condition we establish an exponentially slow lower bound, and both
regimes are empirically supported on BERT and RoBERTa. Long-run chains exhibit metastable trap states lasting hundreds to
thousands of steps and recurrent entry into semantic basins, with
political content serving as a measurable case study of such recurrence. Unlike
single-mask probing or pseudo-log-likelihood scoring, the dynamical framework
exposes global semantic structure that is invisible to existing methods, providing a new framework for understanding global structure and biases in masked language models.

\paragraph{Limitations and Future Work}
Empirically, BERT mixes faster than the bound established in \ref{thm:contraction} predicts. We conjecture this is because the chain rapidly leaves coherent attractors and spends most of its time in incoherent configurations where typical pairwise influence is much smaller than $\alpha(\tau)$. A complete characterization --- for instance via a context-averaged argument that integrates over the stationary distribution rather than taking worst-case suprema --- is left to future work. Relatedly, the cross-over from slow to fast mixing does not seem to sharpen as $n$ increases. Fully explaining this discrepancy would require further mechanistic analysis.

We do not experimentally analyze how incompatibility translates to deviation between the empirical stationary distribution and pseudo-likelihood scores. Additionally, our characterization of recurrent basins relies on zero-shot topic classifiers; finer-grained semantic analysis is left to future work. Finally, we do not provide a closed-form characterization of $\mu_\tau$; preliminary Lyapunov-function approximations rank sequences by occurrence probability but do not match $\mu_\tau$ quantitatively.


\newpage

\bibliographystyle{plainnat}
\bibliography{neurips_2026}

\newpage

\appendix
\section{Technical appendices and supplementary material}Technical appendices with additional results, figures, graphs, and proofs may be submitted with the paper submission before the full submission deadline (see above). You can upload a ZIP file for videos or code, but do not upload a separate PDF file for the appendix. There is no page limit for the technical appendices. Note: Think of the appendix as ``optional reading'' for reviewers. The paper must be able to stand alone without the appendix; for example, adding critical experiments that support the main claims to an appendix is inappropriate.

\subsection{Measuring Token Influence Functions and Score Margins}\label{app:measurements}\

Theorems~\ref{thm:contraction} and~\ref{thm:exponential} are structural: they apply to any chain satisfying their respective hypotheses. We now empirically estimate temperature conditions such that BERT operates in the parameter regimes these theorems describe.

\paragraph{Influence coefficients and the fast-mixing threshold.}
We estimate $\alpha(\tau) = \max_i \sum_{j \neq i} c_{ij}(\tau)$ directly by 
computing TV distances between conditionals at site $i$ as we vary the token at 
site $j$, averaged over multiple base contexts on a 14-token seed sentence. The 
empirical estimate is a lower bound on each $c_{ij}(\tau)$, and hence on 
$\alpha(\tau)$. We obtain $\alpha(10) \gtrsim 1.04$ and $\alpha(11) \gtrsim 0.94$, 
so the contraction sufficient condition $\alpha(\tau) < 1$ requires at least $\tau>10$, roughly an order of magnitude above the empirical fast-mixing 
boundary $\tau \approx 1.5$ (Section~\ref{sec:phase_transition}). The gap reflects that 
$\alpha(\tau)$ is a supremum over both contexts and update sites, dominated by 
coherent contexts where BERT's conditionals are sharply peaked; typical chain 
configurations are less coherent and exhibit smaller effective pairwise influence.

\paragraph{Score margins and trap verification.}Theorem~\ref{thm:exponential} requires a metastable basin $B$ satisfying the uniform local margin condition (Assumption~\ref{ass:margin}). We verify this directly by running 30 Glauber chains of $10{,}000$ steps at $\tau = 1.0$, detecting traps via a rolling embedding-distance criterion (flagging windows where the 300-step mean cosine drift falls below 0.12), and then measuring the per-position score gap $\Delta_i(x) = s_i(x_i;\, x_{-i}) - \max_{a \neq x_i} s_i(a;\, x_{-i})$ for each detected trap state. We find 17 trap states (mean duration 671 steps, range 150--4{,}950), which partition into two types. \textit{Perfect traps} (4 of 17) satisfy Assumption~\ref{ass:margin} exactly:every token is BERT's argmax with $\Delta^* > 0$.The extreme case is a 15-token string of periods (\texttt{...............}),with $\bar{\Delta} = 9.14$, $\Delta^* = 7.46$, and $P(\text{current token}) \geq 0.997$ at every position; this trap persisted for 4{,}950 of 10{,}000 steps. A semantically richer example, ``it has been used as the theme song for jesus christ superstar,'' achieves $\Delta^* = 0.56$ and held for 650 steps. We hypothesize that factual specificity creates strong traps because training data nearly uniquely determines each token given the rest. \textit{Approximate traps} (13 of 17) have at least one position where $\Delta_i < 0$ and thus do not satisfy the sufficient conditions, yet persist for 150--1{,}200 steps. The mechanism is collective stability: strongly-locked surrounding structure maintains metastability even when individual tokens are locally suboptimal. A recurring example is the family of scoring sentences (winners were awarded 2 points, while losers would be awarded 1 point''), where while'' ($\Delta = -1.06$) and would'' ($\Delta = -0.39$) are not argmax choices, yet $\bar{\Delta} = 4.12$ across the full sentence. Similarly, four traps in a single chain share the Wikipedia template the population density was (\ldots) [number] per square mile,'' cycling within this family for over 3{,}000 combined steps. This suggests the energy landscape contains basin \emph{neighborhoods} — clusters of related metastable states — rather than isolated fixed points.

\begin{table}[ht]
\centering
\small
\begin{tabular}{llccccc}
\hline
No. & Representative text & Dur. & $\bar{\Delta}$ & $\Delta^*$ & All argmax & Type \\
\hline
1  & \texttt{...............}                        & 4950 & 9.14 & \phantom{$-$}7.46 & $\checkmark$ & Perfect \\
11 & theme song for jesus christ superstar           &  650 & 5.44 & \phantom{$-$}0.56 & $\checkmark$ & Perfect \\
4  & winners are awarded 2 points\ldots              &  350 & 4.16 & \phantom{$-$}0.06 & $\checkmark$ & Perfect \\
6  & it took me a moment to figure out\ldots         &  350 & 5.78 & \phantom{$-$}0.01 & $\checkmark$ & Perfect \\
\hline
5  & winners were awarded 2 points, while\ldots      & 1200 & 4.12 &            $-1.06$ & $\times$  & Approx. \\
16 & `six months ago, you mean.' `six\ldots'         & 1050 & 4.99 &            $-1.09$ & $\times$  & Approx. \\
10 & population density was (approx.) 7.5\ldots      &  300 & 9.23 &            $-0.22$ & $\times$  & Approx. \\
15 & 9. honest response to you. 10. fan\ldots        &  650 & 3.10 &            $-9.56$ & $\times$  & Approx. \\
\hline
\end{tabular}
\caption{Representative trap states detected across 30 chains ($\tau = 1.0$, $10{,}000$ steps each). Duration in Glauber steps. $\bar\Delta$ = mean score gap; $\Delta^*$ = minimum. Perfect traps satisfy Assumption~\ref{ass:margin}; approximate traps are metastable despite violating it.}
\label{tab:traps}
\end{table}

\paragraph{Tightness of the slow-mixing bound.}
For perfect traps, Theorem~\ref{thm:exponential} predicts 
$t_{\mathrm{mix}}(1/4) \geq \tfrac{1}{4|V|}\,e^{\Delta_*/\tau}$. 
Because $\Delta_*/\tau < \ln|V| \approx 10.3$ for all observed traps, the predicted 
bound falls below one step (Trap~1: predicted $\approx 0.014$, observed $4{,}950$), 
so the bound confirms positive escape time but is not quantitatively informative at 
these gap sizes. The $|V|$ prefactor and the worst-case minimum-gap definition of 
$\Delta_*$ are the main sources of looseness; the mean gap $\bar{\Delta}$ correlates 
better with observed duration empirically but lacks a corresponding rigorous bound.

\section{Additional Experimental Details}\label{app:details}\subsection{Experimental Setup}\label{sec:setup_full}

\paragraph{Model.} We use BERT-base-uncased (Devlin et al., 2019) accessed via the HuggingFace Transformers library.

\paragraph{Seed texts.} Unless otherwise specified, initial sequences $x_0$ are drawn from the MS MARCO passage dataset  truncated to a specified number of tokens $n$.

\paragraph{Embedding model.} For measuring semantic drift, we use Sentence Transformers.

\paragraph{Classifiers.} For topic and sentiment analysis, we use Hugging Face pipelines. For sentiment we use the sentiment pipeline, initialized with pipeline("sentiment-analysis") and for topics we use the classifier pipeline pipeline("zero-shot-classification", model="facebook/bart-large-mnli"), for a list of candidate topics :'politics', 'health', 'pop culture', 'sport', 'food', 'science', 'technology', 'finance', 'education', 'environment', and 'art'.

\paragraph{Hardware.} Experiments were conducted on Google Colab using T4 and A100 GPUs.\paragraph{Experimental protocols.} We use two main protocols for studying mixing time:

\begin{itemize}
\item \textbf{Version 1 (fixed time):} Fix $T$ steps, measure $d(x_T, x_0)$ as a function of sequence length $n$.
\item \textbf{Version 2 (hitting time):} Fix target distance $R$, measure the hitting time $T_R = \min\{t : d(x_t, x_0) \geq R\}$ as a function of $n$.\end{itemize}\subsection{Hyperparameters}Unless otherwise specified:\begin{itemize}
\item Temperature: $\tau = 1.0$
\item Steps: $T = 10^4$
\item Seed text length: 20 tokens
\item Trials per condition: 100\end{itemize}

\subsection{Embedding Model Details}
We use Jina Embeddings v4, accessed via the HuggingFace sentence-transformers library. The model produces 768-dimensional embeddings optimized for semantic similarity.\subsection{Measuring Mixing via Embeddings} We measure drift and mixing using text embeddings. We use a pretrained sentence embedding model $\phi: \mathcal{X} \to \mathbb{R}^d$ (Jina Embeddings v4 with $d = 768$) to project sequences into a continuous semantic space. We define the \emph{embedding distance} from the initial state as:\begin{equation}d(x_t, x_0) = 1 - \cos(\phi(x_t), \phi(x_0))\end{equation} where $\cos(\cdot, \cdot)$ denotes cosine similarity. This embedding technique quantifies the drift from the original meaning, which we call semantic mixing.

\subsection{Endpoint Cross-Model Political-Content Recurrence}
\label{app:cross_model_drift}

\begin{table}[h]
\centering
\small
\caption{Preliminary endpoint cross-model comparison using BART-MNLI.
BERT values use the expanded 10-seed-per-topic experiment; RoBERTa and
ModernBERT values use 5 seeds per topic. Values are mean politics
scores at step 10{,}000. $\checkmark$ marks topics crossing the
baseline of $0.111$.}
\label{tab:cross_model_drift}
\begin{tabular}{lrrr}
\toprule
\textbf{Topic} & \textbf{BERT} & \textbf{RoBERTa} & \textbf{MBERT} \\
\midrule
Health     & 0.044 & \textbf{0.112}$\checkmark$ & 0.026 \\
Sport      & 0.036 & 0.042 & 0.052 \\
Science    & 0.056 & 0.062 & 0.037 \\
Food       & \textbf{0.174}$\checkmark$ & 0.052 & 0.038 \\
Technology & 0.038 & \textbf{0.198}$\checkmark$ & 0.033 \\
\bottomrule
\end{tabular}
\end{table}

\subsection{PCA Analysis and Trajectory Visualization}\label{sec:pca}

We applied two-dimensional PCA to sentence embeddings from the \texttt{SentenceTransformer} model to visualize the temporal evolution of generated sequences, with points colored from warm (early) to cool (late) steps and sequential points connected by lines. The resulting trajectories (Figures~\ref{fig:pca_10000} and \ref{fig:pca_3500}) reveal clear point clusters corresponding to repeated identical sentences, which we interpret as traps in the projected space; note, however, that additional traps may exist in the remaining 766 dimensions not visible in this projection.

\section{Extended Example Trajectories}
\label{app:trajectories}

The examples in this appendix illustrate the qualitative behaviors predicted by the
theoretical results of Section~5. Sections~\ref{sec:traps} and~\ref{sec:temp} ground
the trap and metastability phenomena in concrete outputs. Section~\ref{sec:prefixes}
documents how fixed prefixes create strong semantic attractors, accelerating the
drift described quantitatively in Section~6.3.

\subsection{Characterization of Trap States}
\label{sec:traps}

A striking phenomenon at moderate-to-low temperatures is the emergence of
\textbf{trap states}: configurations where the chain remains stuck for thousands of
steps. Formally, a trap corresponds to a metastable basin $B$ satisfying the local
margin condition of Assumption~1; see Table~5 for the full catalogue of traps
detected across 30 chains, including their mean and minimum score gaps $\bar{\Delta}$
and $\Delta^*$.

\subsection{Examples of Trap States}

We document two categories of traps. In both cases the chain enters the trap
gradually: embedding distance from the initial state stops growing and the
300-step cosine drift falls below 0.12 (the detection threshold of Appendix~A.1),
where it remains until an escape event, typically triggered by a low-probability
token substitution at the weakest-scoring position.

\paragraph{Semantically coherent traps.}
Some traps are grammatically correct and semantically meaningful. These tend to
be \emph{perfect traps} in the sense of Table~5: every token is BERT's argmax
given its context, producing large positive $\Delta^*$. We hypothesize, consistent
with Appendix~A.1, that factual or topical specificity is the key driver: when the
surrounding context nearly uniquely determines each token in the training
distribution, the resulting score margins are large and the basin is deep.

\begin{quote}
\small\texttt{[``guacamole'' trap, steps 2k--10k]}
\textit{In 20th century England, ornamental guacamole was made by filling wooden
sauce jars with vegetables, and served with boiled cigars.}
\end{quote}

\begin{quote}
\small\texttt{[``Brett Kavanaugh'' trap from step 8k]}
\textit{Donald Trump said that he would support Brett Kavanaugh for the Supreme
Court, though he will think about future judicial nominations. But for now, Brett
Kavanaugh is in fact dead.}
\end{quote}

Note that both of these traps are factually wrong or internally contradictory, yet
locally stable: BERT's conditional distributions are shaped by co-occurrence
statistics in training data, not by global factual consistency.

\begin{quote}
\small\texttt{[``kissing modern European men'' trap, steps 6k--10k]}
\textit{The fool always says, don't be jealous of kissing modern European men, but
that's a nonsense statement. You were never meant to be jealous of kissing modern
European men.}
\end{quote}

\begin{quote}
\small\texttt{[``Night of the Living Dead'' trap from step 8k]}
\textit{It's a shame that Night of the Living Dead featured so many creepy aliens
we never knew existed, so who can blame us for asking? What are your two favorite
characters?}
\end{quote}

\paragraph{Nonsensical traps.}
Other traps are grammatically broken yet locally stable --- an instance of the
\emph{approximate trap} regime of Table~5, where collective stability across
the sentence maintains metastability even when individual positions are not
argmax choices. The mechanism is mutual reinforcement: each token is assigned
high conditional probability by the others, even though no coherent global
interpretation exists.

\begin{quote}
\small\texttt{[step 8k]}
\textit{It can be caused by one or more of the three coathoses, the voltage on
top of the trap used to control movement of the contents from one part of the
wall to another.}
\end{quote}

\begin{quote}
\small\texttt{[step 10k]}
\textit{The spin is defined in one or more of the following ways: by decreasing
the periodic table so that no sufficiently large spike exists, or by increasing
the periodic table again through time until a sufficiently large spike is
observed.}
\end{quote}

\subsection{Temperature Regimes and Sentence Coherence}
\label{sec:temp}

At low temperatures ($\tau \leq 0.5$), we observe dramatically slower mixing,
consistent with the slow-mixing regime of Figure~2. The chain frequently becomes
trapped in metastable states for extended periods, and embedding distance from
the initial state (Equation~B.1) grows far more slowly than the $C(\tau)n\log n$
trend observed at high temperature.

For example, at $\tau = 0.1$, chains initialized with a rabies passage remained
within cosine distance ${\approx}0.05$ of the initial embedding for thousands of
steps, and the text preserved much of the original meaning:

\begin{quote}
\small\textbf{[temp 0.1, steps 3k--8k]:} \textit{Population control and vaccination
campaigns have prevented the spread of rabies from spreading to a number of
countries around the world.}
\end{quote}

At $\tau = 0.01$, the text barely changed over 10k steps (cosine distance
$< 0.01$):

\begin{quote}
\small\textbf{[temp 0.01, step 1k to 10k]:} \textit{Birth control and vaccination
programs have prevented the spread of rabies among children in a number of parts
of the world.}
\end{quote}

At very high temperatures ($\tau > 2$), the chain produces nonsensical output as
token selection becomes nearly uniform, destroying all semantic structure:

\begin{quote}
\small\texttt{[temp=10, step 10k]}
\textit{Py specific CNBCagged guarantees easiest reeling coin Lung NATillusionteen
torque KEY a speculateHon Keystone MaidenHideDay Dustin}
\end{quote}

\begin{quote}
\small\texttt{[temp=100, step 10k]}
\textit{HAHAHAHAexpected Downtownasin Company287 vetting vergeardless od
iterationAllows Garage climaxstud oranges symptoms slots moveSTR Christoksovsky}
\end{quote}

\subsection{Prefix-Conditioned Attractors}
\label{sec:prefixes}

A fixed prefix constrains every resampling step: the tokens covered by the prefix
are never updated, biasing the conditional distributions at all other positions.
This effectively narrows the reachable state space and can concentrate probability
mass on a specific semantic basin, accelerating the drift discussed in
Section~6.3. The following examples illustrate how the identity encoded in the
prefix shapes the attractor the chain converges to.

\textbf{``Albert Einstein said that...''}
\begin{itemize}
    \item \textit{...he never expressed his thoughts or feelings to his audience,
    but they always expressed their opinion: everyone needs to do as much good
    work as we can.}
    \item \textit{...general relativity is an extremely tough problem. Because it
    can't be solved by quantum observation alone, it would be quite difficult to
    solve it more broadly.}
\end{itemize}

\textbf{``Elvis Presley said that...''}
\begin{itemize}
    \item \textit{...Bojangles had texted him right before the press conference,
    making it clear to him that they plan to spend as much time together as they
    can.}
\end{itemize}

\textbf{``LeBron James said that...''}
\begin{itemize}
    \item \textit{...with three and a half games to play, `I've got to be able to
    make the impact and get as much playing time as possible.'}
\end{itemize}

\textbf{``The little kid said that...''}
\begin{itemize}
    \item \textit{...the recordings could be used by more people to exploit the
    footage, and added that it might be used as a communications tool for
    terrorists planning another attack.}
\end{itemize}

The last example illustrates that prefix-conditioned attractors are determined by
training co-occurrence, not semantic plausibility: ``the little kid said that''
shares syntactic structure with news-report speech attributions, pulling the chain
toward the same political and crime-adjacent content as directly political
prefixes. This is the same mechanism responsible for the drift documented in
Section~6.3, and the ``Donald Trump said that'' prefix discussed there represents
the extreme case: chains become political almost immediately, with the prefix
providing essentially no barrier to the political attractor.

\section{Supplemental Figures}\label{app:supp}

\begin{figure}[ht]
    \centering
    \includegraphics[width=0.8\textwidth]{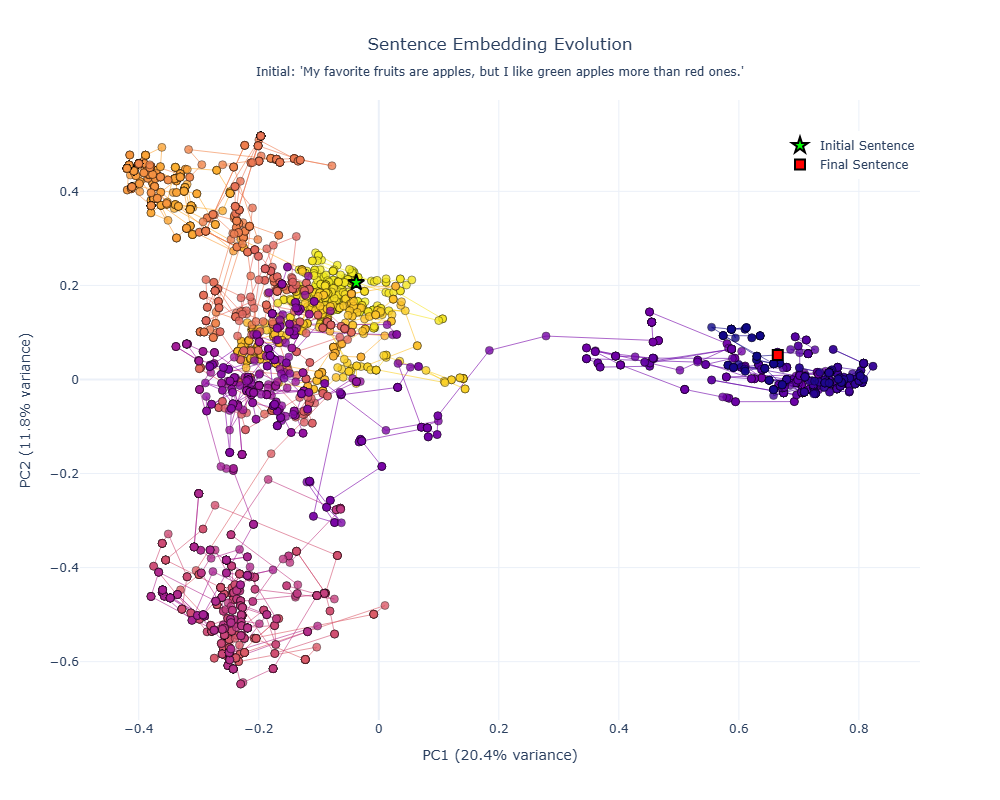}
    \caption{PCA projections of embedding trajectories at 3500 steps. Warm colors indicate early steps, cool colors indicate later steps; clusters correspond to traps in the projected space. \textbf{Initial:} ``My favorite fruits are apples, but I like green apples more than red ones.'' \textbf{Final:} ``But the papers from the Committee on Reconciliation are worth taking a look anyway.''}
    \label{fig:pca_3500}
\end{figure}

\begin{figure}[ht]
    \centering
    \includegraphics[width=0.8\textwidth]{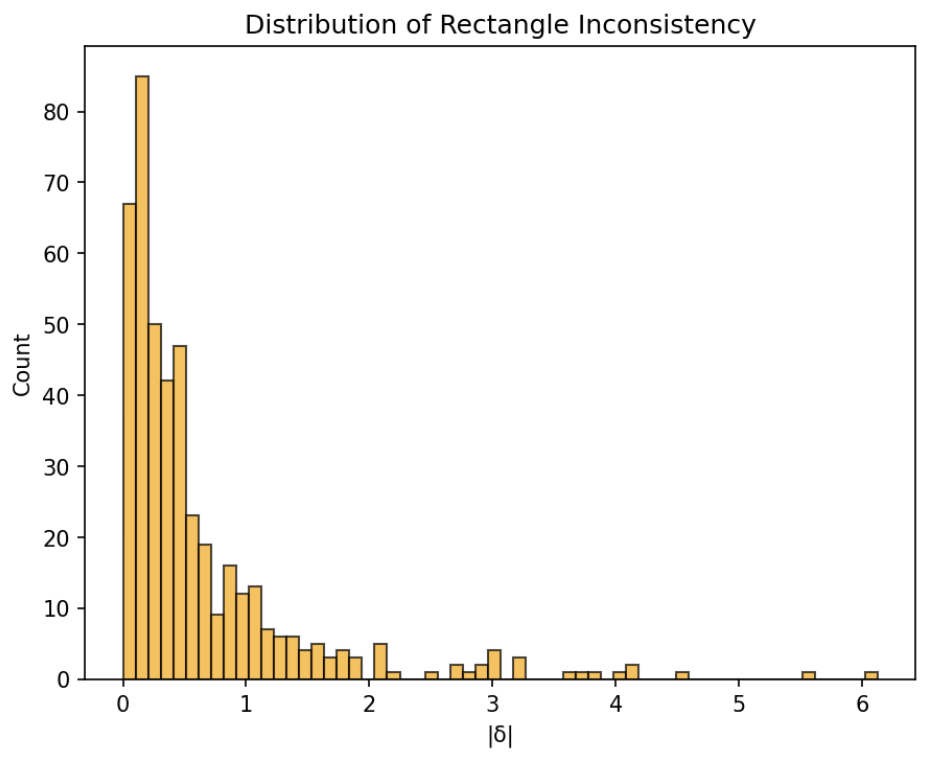}
    \caption{Distribution of rectangle incompatibility with BERT; methodology described in Section \ref{sec:incompatibility}.}
    \label{fig:incompatibility_distribution}
\end{figure}

\begin{figure}[ht]
    \centering
    \includegraphics[width=0.8\textwidth]{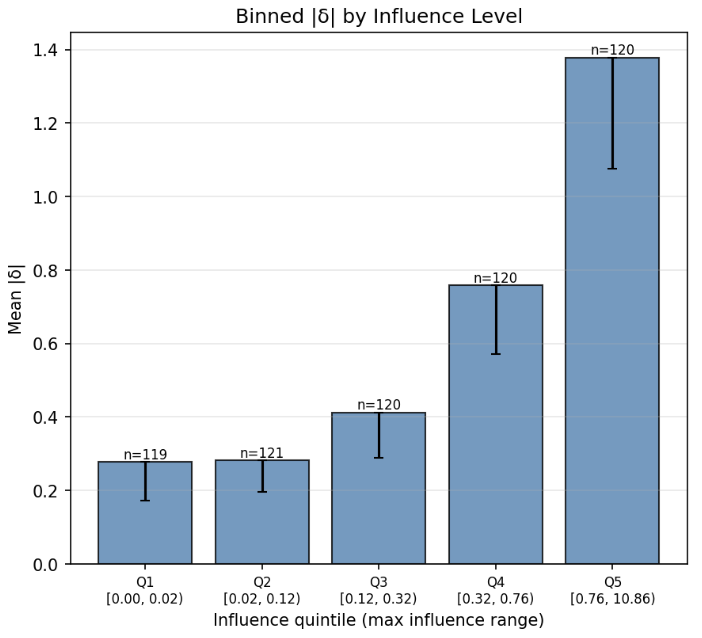}
    \caption{Token-level Influence Amplifies Rectangle Incompatibility. $\text{Influence}(i,j)$ is measured by absolute change in the log-probability of the token at position $i$ when the token at $j$ is swapped; here we take $\max(\text{Influence}(i,j),\text{Influence}(j,i)).$}
    \label{fig:incompatibility_influence}
\end{figure}

\begin{figure}
    \centering
    \includegraphics[width=0.9\linewidth]{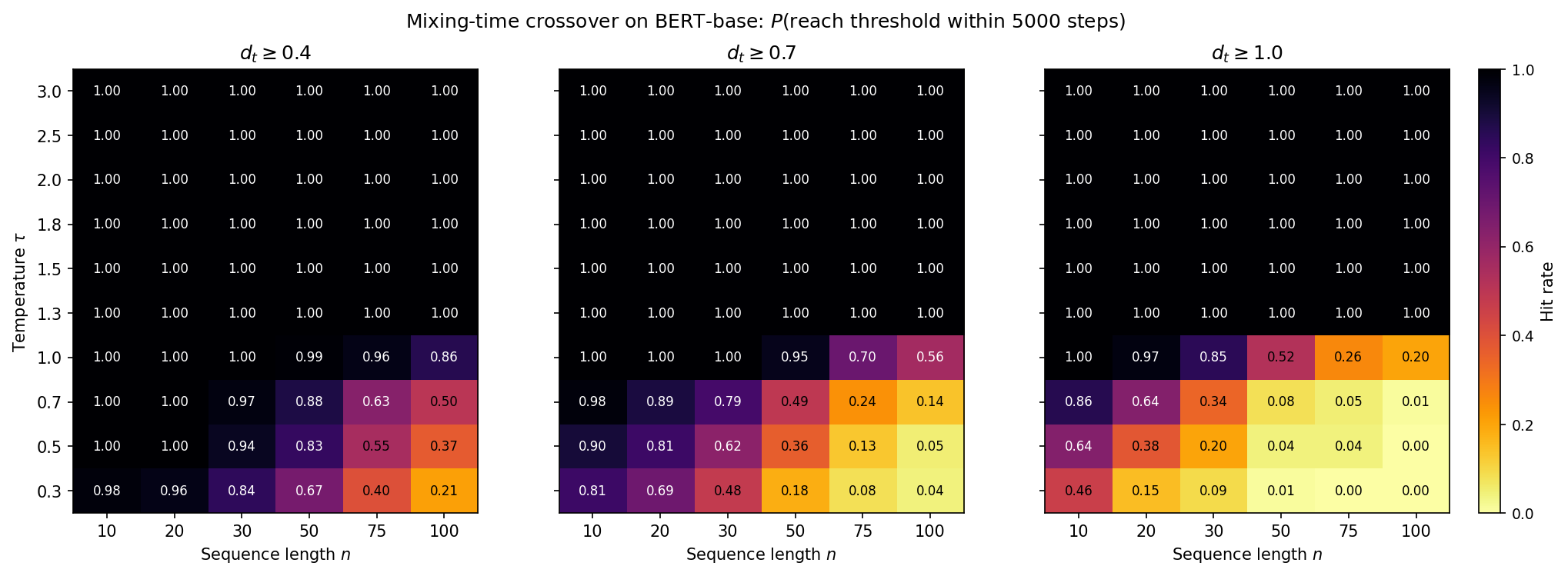}
    \caption{Fraction of 100 initialized chains achieving embedding distance $d_t\geq\{0.4,0.7,1.0\}$ from starting sequence as a function of temperature and sequence length. Mini-LM 384-dimensional embeddings applied every 5 steps; cutoff at 5000 steps.}
    \label{fig:mixing_hitting}
\end{figure}

\section{Deferred Proofs}\label{app:proofs}

\begin{prop}\label{prop:markovian}
For each $\tau>0$, $P_\tau$ is a Markov kernel on $\X$. If $p_{\theta,\tau}(a\mid x_{-i})>0$ for all $i,x_{-i},a$, then $P_\tau$ is irreducible and aperiodic.
\end{prop}
\begin{proof}
Nonnegativity is immediate. For fixed $x\in\X$,$$\sum_{x'\in\X} P_\tau(x,x')
=
\frac1n \sum_{i=1}^n \sum_{x'\in\X}
\mathbf{1}_{x_{-i}=x'_{-i}} p_{\theta,\tau}(x_i'\mid x_{-i}).$$For fixed $i$, the constraint $x'_{-i}=x_{-i}$ leaves only the $i$-th coordinate free, so$$\sum_{x'\in\X}
\mathbf{1}_{x_{-i}=x'_{-i}} p_{\theta,\tau}(x_i'\mid x_{-i})
=
\sum_{a\in\V} p_{\theta,\tau}(a\mid x_{-i}) = 1.$$Hence$$\sum_{x'\in\X} P_\tau(x,x') = \frac1n \sum_{i=1}^n 1 = 1.$$Irreducibility: for any $x,y\in\X$, change the coordinates of $x$ one by one to match $y$. Each prescribed single-site change has positive probability.Aperiodicity: for every $x\in\X$,$$P_\tau(x,x)
=
\frac1n \sum_{i=1}^n p_{\theta,\tau}(x_i\mid x_{-i}) > 0.$$Thus every state has a self-loop.\end{proof}\begin{thm}[High-temperature contraction]Assume $\alpha(\tau)<1$. Then for the Hamming metric$$d_H(x,y):=\sum_{i=1}^n \mathbf{1}_{x_i\neq y_i},$$there exists a coupling of the chain such that for all $x,y\in\X$,$$\mathbb{E}[d_H(X_1,Y_1)\mid X_0=x,Y_0=y]
\le
\left(1-\frac{1-\alpha(\tau)}{n}\right)d_H(x,y).$$Consequently,$$t_{\mathrm{mix}}(\varepsilon)
\le
\frac{n}{1-\alpha(\tau)}
\left(\log n + \log \frac1\varepsilon\right),$$where $t_{\mathrm{mix}}(\varepsilon)$ denotes total-variation mixing time.\end{thm}\begin{proof}Fix $x,y\in\X$. We couple one step as follows: choose the same update site $I\sim \mathrm{Unif}[n]$ in both chains. Conditional on $I=i$, couple the two resampled tokens using a maximal coupling of$$p_{\theta,\tau}(\cdot\mid x_{-i})
\quad\text{and}\quad
p_{\theta,\tau}(\cdot\mid y_{-i}).$$Under a maximal coupling, the probability of disagreement at the updated site is exactly the total variation distance between these two measures.It suffices to analyze the case $d_H(x,y)=1$ by path coupling. Let $x,y$ differ only at site $j$.If $I=j$, then $x_{-j}=y_{-j}$, so the two conditional distributions are identical. Hence under the coupling, the resampled tokens agree almost surely and therefore$$d_H(X_1,Y_1)=0.$$If $I=i\neq j$, then all coordinates except possibly $i$ remain unchanged. Since $x$ and $y$ differ only at site $j$, we have$$\mathbb{P}(X_1^i\neq Y_1^i \mid I=i)
=
\|p_{\theta,\tau}(\cdot\mid x_{-i})-p_{\theta,\tau}(\cdot\mid y_{-i})\|_{\mathrm{TV}}
\le c_{ij}(\tau).$$The original disagreement at site $j$ persists because site $j$ is not updated. Thus$$\mathbb{E}[d_H(X_1,Y_1)\mid I=i]
\le 1 + c_{ij}(\tau).$$Averaging over the update site,$$\mathbb{E}[d_H(X_1,Y_1)]
\le
\frac1n\cdot 0 + \frac1n\sum_{i\neq j} (1+c_{ij}(\tau))
=
1-\frac1n+\frac1n\sum_{i\neq j} c_{ij}(\tau)
\le
1-\frac{1-\alpha(\tau)}{n}.$$So the expected distance contracts by this factor for adjacent pairs. By the Bubley--Dyer path coupling theorem, the same contraction extends to arbitrary pairs:$$\mathbb{E}[d_H(X_t,Y_t)] \le \left(1-\frac{1-\alpha(\tau)}{n}\right)^t d_H(x,y).$$Since total variation distance between two laws is bounded by the coupling disagreement probability, and disagreement probability is bounded by expected Hamming distance,$$\|P_\tau^t(x,\cdot)-P_\tau^t(y,\cdot)\|_{\mathrm{TV}}
\le
\mathbb{P}(X_t\neq Y_t)
\le
\mathbb{E}[d_H(X_t,Y_t)].$$Choosing $y$ distributed from stationarity and using $\max_x d_H(x,y)\le n$, we obtain$$\sup_x \|P_\tau^t(x,\cdot)-\mu_\tau\|_{\mathrm{TV}}
\le
n\left(1-\frac{1-\alpha(\tau)}{n}\right)^t
\le
n\exp\left(-\frac{1-\alpha(\tau)}{n}t\right).$$Thus it suffices to take$$t \ge \frac{n}{1-\alpha(\tau)}\left(\log n + \log\frac1\varepsilon\right).$$\end{proof}\begin{lemma}For all $i\neq j$ and $\tau>0$,$$c_{ij}(\tau) \le \frac{\Delta_{ij}}{4\tau}.$$\end{lemma}\begin{proof}Fix contexts $x_{-i},y_{-i}$ differing only at site $j$, and write$$u_a := s_i(a;x_{-i})/\tau,\qquad v_a:=s_i(a;y_{-i})/\tau.$$Let $\sigma(u),\sigma(v)\in\Delta(\V)$ denote the corresponding softmax distributions. A standard softmax Lipschitz estimate with respect to $\ell_\infty$ perturbations modulo additive constants gives$$\|\sigma(u)-\sigma(v)\|_{\mathrm{TV}}
\le
\frac14\left(\max_a (u_a-v_a)-\min_a (u_a-v_a)\right).$$Multiplying through by $1/\tau$ yields$$\|p_{\theta,\tau}(\cdot\mid x_{-i})-p_{\theta,\tau}(\cdot\mid y_{-i})\|_{\mathrm{TV}}
\le \frac{\Delta_{ij}}{4\tau}.$$Taking the supremum proves the claim.\end{proof}\begin{thm}[Low-temperature exponentially small escape]Assume the uniform local margin condition on a basin $B$. Then for every $x\in B$,$$P_\tau(x,B^c)
\le
|\V| e^{-\Delta_\star/\tau}.$$Consequently the conductance of $B$ satisfies$$\Phi_\tau(B)
:=
\frac{\sum_{x\in B}\mu_\tau(x)P_\tau(x,B^c)}{\mu_\tau(B)}
\le
|\V| e^{-\Delta_\star/\tau}.$$If additionally $\mu_\tau(B)\le \frac12$, then$$t_{\mathrm{mix}}(1/4)
\ge
\frac{1}{4|\V|} e^{\Delta_\star/\tau}.$$\end{thm}\begin{proof}Fix $x\in B$. We first bound the probability of exiting $B$ in one step. A transition can leave $B$ only if the updated site $i$ is one whose change may exit $B$, and then one must sample a token $a\neq x_i$ such that $x^{(i\to a)}\notin B$.For such a site $i$, by the local margin assumption,$$s_i(a;x_{-i}) \le s_i(x_i;x_{-i})-\Delta_\star
\qquad\text{for all exit-causing } a.$$Therefore$$p_{\theta,\tau}(a\mid x_{-i})
=
\frac{e^{s_i(a;x_{-i})/\tau}}
{\sum_{b\in\V} e^{s_i(b;x_{-i})/\tau}}
\le
\frac{e^{(s_i(x_i;x_{-i})-\Delta_\star)/\tau}}{e^{s_i(x_i;x_{-i})/\tau}}
=
e^{-\Delta_\star/\tau}.$$Summing over at most $|\V|-1$ exit-causing symbols and over the choice of site yields$$P_\tau(x,B^c)
\le
\frac1n \sum_{i=1}^n (|\V|-1)e^{-\Delta_\star/\tau}
\le
|\V| e^{-\Delta_\star/\tau}.$$Now average with respect to $\mu_\tau$ restricted to $B$:$$\Phi_\tau(B)
=
\frac{1}{\mu_\tau(B)}\sum_{x\in B}\mu_\tau(x)P_\tau(x,B^c)
\le
|\V| e^{-\Delta_\star/\tau}.$$Finally, by the standard conductance lower bound on mixing time for finite Markov chains, if $\mu_\tau(B)\le 1/2$ then$$t_{\mathrm{mix}}(1/4) \ge \frac{1}{4\Phi_\tau(B)}.$$Substituting the conductance estimate gives$$t_{\mathrm{mix}}(1/4)
\ge
\frac{1}{4|\V|}e^{\Delta_\star/\tau}.$$\end{proof} 


\end{document}